\documentclass[preprint,3p, times]{elsarticle}

\usepackage{amssymb}
\usepackage{amsmath}
\usepackage{amsthm}
\usepackage{mathtools}

\newtheorem{theorem}{Theorem}[section]
\newtheorem{lemma}[theorem]{Lemma}
\newtheorem{proposition}[theorem]{Proposition}
\newtheorem{corollary}[theorem]{Corollary}

\journal{Statistics $\&$ Probability Letters}

\begin{document}

\begin{frontmatter}



\title{Edgeworth Expansion for Semi-hard Triplet Loss}


\author{Masanari Kimura} 

\affiliation{organization={School of Mathematics and Statistics, The University of Melbourne},
            addressline={Parkville}, 
            city={Melbourne},
            postcode={3010}, 
            state={VIC},
            country={Australia}}

\begin{abstract}
We develop a higher-order asymptotic analysis for the semi-hard triplet loss using the Edgeworth expansion.
It is known that this loss function enforces that embeddings of similar samples are close while those of dissimilar samples are separated by a specified margin.
By refining the classical central limit theorem, our approach quantifies the impact of the margin parameter and the skewness of the underlying data distribution on the loss behavior.
In particular, we derive explicit Edgeworth expansions that reveal first-order corrections in terms of the third cumulant, thereby characterizing non-Gaussian effects present in the distribution of distance differences between anchor-positive and anchor-negative pairs.
Our findings provide detailed insight into the sensitivity of the semi-hard triplet loss to its parameters and offer guidance for choosing the margin to ensure training stability.
\end{abstract}




\end{frontmatter}

\section{Introduction}
Metric learning, which seeks to learn a meaningful embedding of data into a lower-dimensional space, has numerous applications, such as face recognition~\citep{chopra2005learning,guillaumin2009you,yi2014deep}, image retrieval~\citep{hoi2010semi,gao2014soml,kimura2023decomposition}, and clustering~\citep{xing2002distance,bilenko2004integrating,weinberger2009distance}.
A central concept in metric learning is the triplet loss~\citep{weinberger2009distance,hermans2017defense,dong2018triplet,vygon2021learning}, which simultaneously minimizes the distance between embeddings of similar samples (anchor-positive pairs) and maximizes the distance between dissimilar samples by enforcing a margin.
Although simple in structure, the performance of the triplet loss heavily depends on the choice of the margin parameter and the manner in which negative examples are selected.
In practice, the semi-hard negative mining strategy~\citep{schroff2015facenet,parkhi2015deep} is often employed to select informative triplets.
Specifically, semi-hard negatives are those for which the embedding distance is greater than that of the positive but still within the margin $\alpha$.
This mining strategy effectively excludes trivial negatives, which do not contribute significantly to training, and overly hard negatives that can destabilize the optimization process.
Despite widespread empirical success, there is limited theoretical analysis on how the margin parameter $\alpha$ and the underlying data distribution impact the learning dynamics and performance of the resulting embeddings.

This paper addresses this theoretical gap by conducting a rigorous higher-order asymptotic analysis of the semi-hard triplet loss via the Edgeworth expansion~\citep{kendall1946advanced,kolassa2006series}.
We explicitly derive expansions to show how the margin parameter and underlying data distribution affect loss behavior.
The resulting insights provide theoretical understanding of semi-hard triplet loss, and also offer practical guidance for selecting the margin parameter to enhance training stability and embedding quality.

\section{Preliminary}
Let $\mathcal{X}$ be the input space and let $f \colon \mathcal{X} \to \mathbb{R}^k$ be a trainable embedding function into $k$-dimensional space.
We assume that we have i.i.d. draws of $(X_a, X_p, X_n)$ from some distribution $\mathcal{D}$ over triplets, where $X_a$ is an anchor sample, $X_p$ is a positive sample (same label as $X_a$) and $X_n$ is a negative sample (different label from $X_a$).
Let $d \colon \mathbb{R}^k\times\mathbb{R}^k \to \mathbb{R}_+$ be a differentiable distance function.
The standard triplet loss with margin $\alpha > 0$ is defined as
\begin{equation}
    \mathcal{L}_{\text{triplet}} \coloneqq \mathbb{E}\left[\max\left(\alpha - \Delta, 0\right)\right],
\end{equation}
where $\Delta \coloneqq \Delta(X_a, X_p, X_n) \coloneqq d(f(X_a), f(X_n)) - d(f(X_a), f(X_p))$.
Semi-hard negative mining seeks negative samples $X_n$ that are not too easy and not too hard.
That is, those $X_n$ satisfy
\begin{equation}
    d(f(X_a), f(X_p)) < d(f(X_a), f(X_n)) < d(f(X_a), f(X_p)) + \alpha.
\end{equation}
Equivalently, the semi-hard condition is $0 < \Delta(X_a, X_p, X_n) < \alpha$.
Define the event
\begin{equation}
    E_{\text{sh}} \coloneqq \left\{(X_a, X_p, X_n)\ \colon\ 0 < \Delta(X_a, X_p, X_n) < \alpha\right\}.
\end{equation}
In practice, we only compute or backpropagate the loss for triplets in $E_{\text{sh}}$. 
Formally, one can write the semi-hard triplet loss as
\begin{equation}
    \mathcal{L}_{\text{semi}} \coloneqq \mathbb{E}\left[1_{E_{\text{sh}}} \cdot \max(\alpha - \Delta, 0)\right].
\end{equation}
Note that if $\Delta \geq \alpha$ the loss is zero anyway, and if $\Delta \leq 0$ (an easy case) it also contributes zero.
Hence, focusing on semi-hard triples means focusing on $\Delta \in (0, \alpha)$.

Suppose that $\Delta$ is a well-defined random variable with finite moments.
Let $\mu_\Delta \coloneqq \mathbb{E}[\Delta]$, $\sigma^2_\Delta \coloneqq \mathrm{Var}(\Delta)$, and $\bar{\Delta}_N$ be the sample mean of i.i.d. copies $\{\Delta\}^N_{i=1}$, where $N$ is the batch size.
The central limit theorem (CLT) implies that $\sqrt{N}(\bar{\Delta}_N - \mu_\Delta) \leadsto \mathcal{N}(0, \sigma^2_\Delta)$.
However, higher-order asymptotic expansions enable us to quantify how $\bar{\Delta}_N$ behaves beyond the leading Gaussian approximation.

\section{Main Results}
We provide statements capturing the role of the margin $\alpha$, the underlying distribution of distances, and how the semi-hard event constrains $\Delta$.
\begin{lemma}
    \label{lem:edgeworth_expansion_delta}
    Assume $\Delta$ has finite moments up to order $s \geq 4$.
    Denote by $\kappa_j$ the $j$-th cumulant of $\Delta$.
    Let $Z_N \coloneqq \sqrt{N}(\bar{\Delta}_N - \mu_\Delta) / \sigma_\Delta$
    be the standard average of i.i.d copies of $\Delta$.
    Then, there is an Edgeworth expansion for the distribution function $F_{Z_N}(z)$ of the form
    \begin{align*}
        F_{Z_N}(z) = \Phi(z) - \phi(z)\frac{\gamma_3}{6\sqrt{N}}(z^2 -1) + O(N^{-1}),
    \end{align*}
    where $\Phi$ is the standard normal c.d.f., $\phi$ is the standard normal p.d.f., and $\gamma_3 = \kappa_3 / \kappa_2^{3/2}$ is the skewness.
\end{lemma}
Although Lemma~\ref{lem:edgeworth_expansion_delta} is classical, it lays the foundation for the expansions needed to analyze the triplet loss.
Define the semi-hard probability as
\begin{equation}
    P_{\text{sh}} \coloneqq \mathbb{P}(0 < \Delta < \alpha) = \mathbb{E}\left[1_{\{0<\Delta<\alpha\}}\right].
\end{equation}
Since $\Delta$ is a real-valued random variable, the behavior of $P_{\text{sh}}$ as $\alpha$ varies is crucial to the learning dynamics.
For $\alpha$ near $\mu_\Delta$, one can write $P_{\text{sh}}(\alpha) \coloneqq \int^\alpha_0 f_\Delta(t) dt \approx F_\Delta(\alpha) - F_\Delta(0)$, where $F_\Delta$ is the c.d.f. of $\Delta$.
Then, applying the Edgeworth expansion in Lemma~\ref{lem:edgeworth_expansion_delta} yields the following proposition.
\begin{proposition}
    \label{prp:expansion_p_sh}
    If $\alpha$ and $0$ lie in regions where the Edgeworth expansion applies, then
    \begin{equation}
        P_{\text{sh}}(\alpha) \approx \left[\Phi(\zeta_\alpha) - \Phi(\zeta_0)\right] - \frac{\gamma_3}{6\sqrt{N}}\left[\phi(\zeta_\alpha)\tilde{P}_{1,\alpha} - \phi(\zeta_0)\tilde{P}_{1,0}\right],
    \end{equation}
    where $\zeta_x = (x - \mu_\Delta) / \sigma_\Delta$, and $\tilde{P}_{1,\alpha}$, $\tilde{P}_{1,0}$ are polynomial terms in $\zeta_\alpha$, $\zeta_0$.
\end{proposition}
Proposition~\ref{prp:expansion_p_sh} implies that, for large $N$, $P_{\text{sh}}(\alpha)$ depends on how $\alpha$ compares with the mean $\mu_\Delta$.
The margin $\alpha$ effectively shifts the region $(0, \alpha)$ and the leading correction to the normal approximation is on the order of $1 / \sqrt{N}$ times a function of the skewness $\gamma_3$.
We now combine the expansions to analyze the expected semi-hard triplet loss.
\begin{equation}
    \mathcal{L}_{\text{semi}}(\alpha) = \mathbb{E}\left[1_{\{0 < \Delta < \alpha\}}\cdot \max(\alpha - \Delta, 0)\right] = \int^\alpha_0 (\alpha - t)f_\Delta(t)dt. \label{eq:semi_hard_triplet_los}
\end{equation}
Here, $f_\Delta(t)$ is the p.d.f. of $\Delta$.
\begin{theorem}
    \label{thm:main_theorem}
    Under the assumptions in Lemma~\ref{lem:edgeworth_expansion_delta}, the semi-hard triplet loss in Eq.~\eqref{eq:semi_hard_triplet_los} admits the Edgeworth expansion,
    \begin{align*}
        \mathcal{L}_{\text{semi}}(\alpha) = \mathcal{L}^{(0)}_\alpha + \frac{1}{\sqrt{N}}\mathcal{L}^{(1)}_\alpha + O(N^{-1}),
    \end{align*}
    with
    \begin{align*}
        \mathcal{L}^{(0)}_\alpha &= (\alpha - \mu_\Delta)\left[\Phi_{\zeta_\alpha} - \Phi_{\zeta_0}\right] + \sigma_\Delta\left[\phi_{\zeta_\alpha} - \phi_{\zeta_0}\right], \\
        \mathcal{L}^{(1)}_\alpha &= \frac{\gamma_3}{6}\left\{(\alpha - \mu_\Delta)\left[\phi_{\zeta_\alpha}\zeta_\alpha^2 - \phi_{\zeta_0}\zeta_0^2 - \phi_{\zeta_\alpha} + \phi_{\zeta_0}\right] - \sigma_\Delta\left[\zeta_\alpha^3\phi_{\zeta_\alpha} - \zeta_0^3\phi_{\zeta_0}\right]\right\},
    \end{align*}
    where $f_\mathcal{N}(t)$ is the density of $\mathcal{N}(\mu_\Delta, \sigma_\Delta^2)$, $\Phi_{\zeta_t} = \Phi(\zeta_t)$, $\phi_{\zeta_t} = \phi(\zeta_t)$ and $\zeta_t = (t - \mu_\Delta) / \sigma_\Delta$.
\end{theorem}
\begin{proof}
    Introduce the change of variable $\zeta_t = (t - \mu_\Delta) / \sigma_\Delta$ so that $t = \mu_\Delta + \sigma_\Delta \zeta_t$ and $dt = \sigma_\Delta d\zeta_t$.
    Then, the loss becomes
    \begin{align*}
        \mathcal{L}_{\text{semi}}(\alpha) &= \int^{t=\alpha}_{t=0}(\alpha - t)f_\Delta(t)dt \\
        &= \int^{\zeta_t=\zeta_\alpha}_{\zeta_t=\zeta_0}\left[\alpha - (\mu_\Delta + \sigma_\Delta \zeta_t)\right]f_\Delta(\mu_\Delta + \sigma_\Delta \zeta_t)\sigma_\Delta d\zeta_t \\
        &= \int^{\zeta_t=\zeta_\alpha}_{\zeta_t=\zeta_0}\left[\alpha - \mu_\Delta - \sigma_\Delta \zeta_t\right]\phi(\zeta_t)\left[1 - \frac{\gamma_3}{6\sqrt{N}}(\zeta_t^3 - 3\zeta_t) + O(N^{-1})\right]d\zeta_t \\
        &= I_0 - \frac{\gamma_3}{6\sqrt{N}}I_1 + O(N^{-1}),
    \end{align*}
    where
    \begin{align*}
        I_0 &\coloneqq \int^{\zeta_t = \zeta_\alpha}_{\zeta_t = \zeta_0}\left[\alpha - \mu_\Delta - \sigma_\Delta \zeta_t \right]\phi(\zeta_t) d\zeta_t, \\
        I_1 &\coloneqq \int^{\zeta_t = \zeta_\alpha}_{\zeta_t = \zeta_0} \left[\alpha - \mu_\Delta - \sigma_\Delta \zeta_t \right](\zeta_t^3 - 3\zeta_t)\phi(\zeta_t)d\zeta_t.
    \end{align*}
    First, we consider to evaluate $I_0$ as follows.
    \begin{align*}
        I_0 &= (\alpha - \mu_\Delta)\int^{\zeta_t = \zeta_\alpha}_{\zeta_t = \zeta_0}\phi(\zeta_t)d\zeta_t - \sigma_\Delta\int^{\zeta_t = \zeta_\alpha}_{\zeta_t = \zeta_0} \zeta_t \phi(\zeta_t)d\zeta_t \\
        &= (\alpha - \mu_\Delta)\left[\Phi(\zeta_\alpha) - \Phi(\zeta_0)\right] + \sigma_\Delta\left[\phi(\zeta_\alpha) - \phi(\zeta_0)\right].
    \end{align*}
    We next evaluate $I_1$.
    Let $I_1 = (\alpha - \mu_\Delta)J_1 - \sigma_\Delta J_2$,
    with
    \begin{align*}
        J_1 &= \int^{\zeta_t = \zeta_\alpha}_{\zeta_t = \zeta_0}(\zeta_t^3 - 3\zeta_t) \phi(\zeta_t)d\zeta_t, \\
        J_2 &= \int^{\zeta_t = \zeta_\alpha}_{\zeta_t = \zeta_0} \zeta_t(\zeta_t^3 - 3\zeta_t)\phi(\zeta_t)d\zeta_t = \int^{\zeta_t = \zeta_\alpha}_{\zeta_t = \zeta_0} \zeta_t^4 - 3\zeta_t^2 \phi(\zeta_t)d\zeta_t.
    \end{align*}
    Here,
    \begin{align*}
        J_1 &= \int^{\zeta_t = \zeta_\alpha}_{\zeta_t = \zeta_0} \zeta_t^3 \phi(\zeta_t) d\zeta_t - 3 \int^{\zeta_t = \zeta_\alpha}_{\zeta_t = \zeta_0} \zeta_t \phi(\zeta_t)d\zeta_t \\
        &= \Big[-\phi(\zeta_t)(\zeta_t^2 + 2)\Big]^{\zeta_\alpha}_{\zeta_0} - 3\Big[-\phi(\zeta_t)\Big]^{\zeta_\alpha}_{\zeta_0} \\
        &= \left\{-\phi(\zeta_\alpha)(\zeta_\alpha^2 + 2) + \phi(\zeta_0)(\zeta_0^2 + 2)\right\} + 3\left[\phi(\zeta_\alpha) - \phi(\zeta_0)\right] \\
        &= -\phi(\zeta_\alpha)\zeta_\alpha^2 + \phi(\zeta_0)\zeta_0^2 + \phi(\zeta_\alpha) - \phi(\zeta_0),
    \end{align*}
    and
    \begin{align*}
        J_2 &= \int^{\zeta_t = \zeta_\alpha}_{\zeta_t = \zeta_0} \zeta_t^4 \phi(\zeta_t) d\zeta_t - 3\int^{\zeta_t = \zeta_\alpha}_{\zeta_t = \zeta_0} \zeta_t^2 \phi(\zeta_t)d\zeta_t \\
        &= \Big[-\zeta_\alpha^3\phi(\zeta_\alpha) - 3\zeta_\alpha\phi(\zeta_\alpha) + 3\Phi(\zeta_\alpha) - \Big(-\zeta_0^3\phi(\zeta_0) - 3\zeta_0\phi(\zeta_0) + 3\Phi(\zeta_0)\Big)\Big] \\
        &\quad\quad\quad - 3\Big\{\Big[-\zeta_\alpha\phi(\zeta_\alpha) + \Phi(\zeta_\alpha)\Big] - \Big[-\zeta_0\phi(\zeta_0) + \Phi(\zeta_0)\Big] \Big\} \\
        &= \Big[-\zeta_\alpha^3\phi(\zeta_\alpha) - 3\zeta_\alpha\phi(\zeta_\alpha) + 3\Phi(\zeta_\alpha) \zeta_0^3\phi(\zeta_0) + 3\zeta_0\phi(\zeta_0) - 3\Phi(\zeta_0)\Big] \\
        &\quad\quad\quad + 3\Big\{\zeta_\alpha\phi(\zeta_\alpha) - \Phi(\zeta_\alpha) - \zeta_0\phi(\zeta_0) + \Phi(\zeta_0) \Big\} = -\zeta_\alpha^3 \phi_{\zeta_\alpha} + \zeta_0^3 \phi_{\zeta_0}.
    \end{align*}
    Thus,
    \begin{align*}
        I_1 &= (\alpha - \mu_\Delta)\Big[-\phi_{\zeta_\alpha}\zeta_\alpha^2 + \phi_{\zeta_0}\zeta_0^2 + \phi_{\zeta_\alpha} - \phi_{\zeta_0}\Big] - \sigma_\Delta \Big[-\zeta_\alpha^3 \phi_{\zeta_\alpha} + \zeta_0^3 \phi_{\zeta_0}\Big],
    \end{align*}
    and we obtain the final expansion.
\end{proof}
\begin{corollary}
    Under the assumptions of Theorem~\ref{thm:main_theorem}, if the distance difference $\Delta$ is symmetrically distributed about its mean, then the first-order Edgeworth correction vanishes.
\end{corollary}
\begin{corollary}[Margin Sensitivity of the Semi-hard Triplet Loss]
    Under the assumptions of Theorem~\ref{thm:main_theorem}, note that the semi-hard triplet loss satisfies
    \begin{equation}
        \frac{d\mathcal{L}_{\text{semi}}(\alpha)}{d\alpha} = \int_0^\alpha f_\Delta(t)\, dt = P_{\text{sh}}(\alpha) = \mathbb{P}(0 < \Delta < \alpha).
    \end{equation}
    Employing the Edgeworth expansion from Proposition~\ref{prp:expansion_p_sh}, we have
    \begin{equation}
        \frac{d\mathcal{L}_{\text{semi}}(\alpha)}{d\alpha} \approx \left[\Phi\Big(\frac{\alpha-\mu_\Delta}{\sigma_\Delta}\Big)-\Phi\Big(\frac{-\mu_\Delta}{\sigma_\Delta}\Big)\right] - \frac{\gamma_3}{6\sqrt{N}}\left[\phi\Big(\frac{\alpha-\mu_\Delta}{\sigma_\Delta}\Big)\tilde{P}_{1,\alpha} - \phi\Big(\frac{-\mu_\Delta}{\sigma_\Delta}\Big)\tilde{P}_{1,0}\right] + O(N^{-1}).
    \end{equation}
    In particular, when the margin is set equal to the mean, the sensitivity becomes
    \begin{equation}
        \left.\frac{d\mathcal{L}_{\text{semi}}(\alpha)}{d\alpha}\right|_{\alpha=\mu_\Delta} \approx \frac{1}{2} - \Phi\Big(\frac{-\mu_\Delta}{\sigma_\Delta}\Big) - \frac{\gamma_3}{6\sqrt{N}}\left[\phi(0)\tilde{P}_{1,\mu_\Delta} - \phi\Big(\frac{-\mu_\Delta}{\sigma_\Delta}\Big)\tilde{P}_{1,0}\right] + O(N^{-1}).
    \end{equation}
    This expression quantifies how small variations in $\alpha$ impact the loss via the probability mass in the semi-hard region.
    Consequently, it provides a direct criterion for tuning $\alpha$ to balance training stability and effective mining of informative triplets.
\end{corollary}
Finally, following statements provide the uniform error bound for the Edgeworth approximation.
\begin{theorem}
    Under the assumptions of Theorem~\ref{thm:main_theorem} and further assuming that the density $f_\Delta(t)$ and its derivatives up to order four are uniformly bounded on the interval $[0,\alpha]$ for $\alpha$ in a compact interval containing $\mu_\Delta$, there exists a constant $C > 0$, independent of $N$, such that for all $\alpha$ in that interval and for all sufficiently large $N$, the approximation error satisfies
    \begin{equation}
        \left|\mathcal{L}_{\text{semi}}(\alpha) - \left(\mathcal{L}^{(0)}_\alpha + \frac{1}{\sqrt{N}}\mathcal{L}^{(1)}_\alpha\right)\right| \le \frac{C}{N}.
    \end{equation}
\end{theorem}
\begin{proof}
    Under the assumptions of Theorem~\ref{thm:main_theorem}, the density $f_\Delta(t)$ admits an Edgeworth expansion of the form
    \begin{equation*}
        f_\Delta(t) = f_{\mathcal{N}}(t) + \frac{1}{\sqrt{N}}\, g(t) + R_N(t),
    \end{equation*}
    where $f_{\mathcal{N}}(t)$ is the density of $\mathcal{N}(\mu_\Delta,\sigma^2_\Delta)$, $(g(t)$ is the first-order correction term, and $R_N(t)$ is the remainder term.
    By our additional assumptions, there exists a constant $K>0$ such that
    \begin{equation*}
        \sup_{t\in[0,\alpha]} \left|R_N(t)\right| \le \frac{K}{N}.
    \end{equation*}
    Thus, we can write
    \begin{equation*}
        \mathcal{L}_{\text{semi}}(\alpha) = \int_0^\alpha (\alpha-t) \left[f_{\mathcal{N}}(t) + \frac{1}{\sqrt{N}}\, g(t)\right]dt + \int_0^\alpha (\alpha-t)R_N(t)\, dt.      
    \end{equation*}
    Define
    \begin{align*}
        \mathcal{L}^{(0)}_\alpha \coloneqq \int_0^\alpha (\alpha-t) f_{\mathcal{N}}(t) dt, \quad \mathcal{L}^{(1)}_\alpha \coloneqq \int_0^\alpha (\alpha-t) g(t) dt.
    \end{align*}
    Then, the total approximation error is given by
    \begin{equation*}
        \mathcal{R}_N(\alpha) = \int_0^\alpha (\alpha-t) R_N(t)\, dt.
    \end{equation*}
    Since for any $t \in [0,\alpha]$, we have $\alpha-t \leq \alpha$ and because $\alpha$ is taken from a compact interval $[0,M]$ with $M>0$ a constant, it follows that
    \begin{equation*}
        \left|\mathcal{R}_N(\alpha)\right| \le \int_0^\alpha (\alpha-t) \left|R_N(t)\right| dt \le \frac{K}{N}\int_0^\alpha (\alpha-t) dt = \frac{K}{N}\cdot \frac{\alpha^2}{2} \le \frac{K M^2}{2N}.
    \end{equation*}
    Defining $C \coloneqq K M^2 / 2$ yields the desired uniform bound for all $\alpha \in [0,M]$ and for all sufficiently large $N$.
\end{proof}
\begin{corollary}
    Under the conditions of Theorem~\ref{thm:main_theorem} and the uniform error bound established above, if the batch size $N$ is chosen such that $C / N \le \epsilon$ for a prescribed tolerance $\epsilon > 0$, then the Edgeworth expansion approximates the semi-hard triplet loss uniformly with relative error not exceeding $\epsilon$.
    That is, for all $\alpha$ in the specified interval,
    \begin{equation}
        \left|\frac{\mathcal{L}_{\text{semi}}(\alpha) - \Bigl(\mathcal{L}^{(0)}_\alpha + \frac{1}{\sqrt{N}}\mathcal{L}^{(1)}_\alpha\Bigr)}{\mathcal{L}_{\text{semi}}(\alpha)}\right| \le \epsilon.
    \end{equation}
\end{corollary}
These results not only extend the theoretical analysis of the semi-hard triplet loss but also provide actionable criteria for practitioners to ensure that the approximation remains within acceptable bounds during training.

\section{Concluding Remarks}
\begin{itemize}
    \item If $\alpha$ is too large, $\Delta \in (0, \alpha)$ becomes easier to satisfy, and hence $\mathcal{L}_{\text{semi}}$ can remain significant even if $\Delta$ is moderately large.
    The expansions show that the first-order sensitivity to $\alpha$ is captured by integrals of normal p.d.f. in a region around $\mu_\Delta$.
    If $\alpha$ is small, the region $(0, \alpha)$ shrinks, so fewer triplets qualify as semi-hard.
    The expansions highlight how $\alpha$ interacts with $\mu_\Delta$ in the integrals, and how the distribution’s skewness $\gamma_3$ can shift the effective mass in $(0, \alpha)$.
    \item The distribution of $\Delta$ depends on the distribution of $(X_a, X_p, X_n)$ and on the embedding function $f$.
    If $f$ is random or high-dimensional, advanced tools from random matrix theory or spherical integrals might further refine these expansions.
    Also, higher-order cumulants become crucial for capturing non-Gaussian effects (e.g., heavy tails or multi-modal behaviors in distance distributions).
    \item These expansions can guide the choice of $\alpha$.
    For example, if expansions suggest that $\mu_\Delta \approx \alpha$, then small changes in $\alpha$ produce large changes in $\mathcal{L}_{\text{semi}}$.
    This can affect gradient magnitudes and training stability.
    In addition, the skewness $\gamma_3$ can indicate whether the distribution of $\Delta$ is long-tailed to the right or left.
    A large positive skew (long right tail) means that many negative samples are actually quite far from the anchor, potentially overshadowing the semi-hard region.
\end{itemize}



\bibliographystyle{elsarticle-num} 
\bibliography{main}

\end{document}